%% file: main.tex
\documentclass[letterpaper, 10 pt, conference]{ieeeconf}
\IEEEoverridecommandlockouts
\input{includes/0-preamble}

\title{\LARGE \bf
Optimal Shelf Arrangement to Minimize Robot Retrieval Time
}

\author{Lawrence Yunliang Chen$^1$, Huang Huang$^1$, Michael Danielczuk$^1$, Jeffrey Ichnowski$^1$, Ken Goldberg$^1$
\thanks{$^1$The AUTOLab at UC Berkeley
, {\tt\small \{yunliang.chen, huangr, mdanielczuk, jeffi,  goldberg\}@berkeley.edu}}
}

\begin{document}

\maketitle
\thispagestyle{empty}
\pagestyle{empty}

\input{includes/0-abstract}
\input{includes/1-introduction}
\input{includes/2-related-work}
\input{includes/3-problem-statement}
\input{includes/4-shelf-density}
\input{includes/5-mip-formulation}
\input{includes/6-mip-analysis}

\input{includes/7-sim-experiments}
\input{includes/8-mech-search}
\input{includes/9-conclusion}
\input{includes/10-acknowledgments}

{\footnotesize
\bibliographystyle{IEEEtran}
\bibliography{references}
}

\clearpage
\input{appendix.tex}

\end{document}

%% file: includes/0-preamble.tex
\usepackage{url}

\usepackage{tabu}
\usepackage{booktabs} 
\usepackage{multirow}
\usepackage{multicol}
\usepackage{adjustbox}

\usepackage{xcolor}
\usepackage{xspace}
\usepackage{ifthen}

\usepackage{graphicx}
\usepackage[font=small]{caption}
\usepackage{tikz}
\usetikzlibrary{positioning}

\let\labelindent\relax 
\usepackage{enumitem} 

\usepackage{autolabtools}
\usepackage{amsmath,amssymb,amsfonts}
\usepackage{dsfont}
\usepackage{color}
\usepackage{verbatim}
\usepackage{subcaption}
\usepackage{balance}
\usepackage{gensymb}
\usepackage{siunitx} 
\usepackage{algorithm}
\usepackage{algorithmicx}
\usepackage{algpseudocode}
\usepackage{alphalph}
\usepackage{comment}
\usepackage{soul}

\graphicspath{{figures/}}

\newcommand{\ie}{i.e., }
\newcommand{\eg}{e.g., }

\newcommand{\R}{\mathbb{R}}

\newcommand{\Z}{\mathbb{Z}}

\newtheorem{theorem}{Theorem}

\allowdisplaybreaks[4]

\DeclareMathOperator*{\argmin}{arg\,min}

\definecolor{britishracinggreen}{rgb}{0.23, 0.53, 0.19}

\definecolor{navy}{rgb}{0,0,0.5}

\usepackage[font={small}]{caption}
\def\tablename{Table}

\newcommand{\algabbr}{OSA-MIP\xspace}

%% file: includes/0-abstract.tex
\begin{abstract}
Shelves are commonly used to store objects in homes, stores, and warehouses.  
We formulate the problem of Optimal Shelf Arrangement (OSA), where the goal is to optimize the arrangement of objects on a shelf for access time given an access frequency and movement cost 
for each object. We propose \algabbr, a mixed-integer program (MIP), show that it finds an optimal solution for OSA under certain conditions, and provide bounds on its suboptimal solutions in general cost settings. We analytically characterize a necessary and sufficient shelf density condition for which there exists an arrangement such that any object can be retrieved without removing objects from the shelf. Experimental data from 1,575 simulated shelf trials and 54 trials with a physical Fetch robot equipped with a pushing blade and suction grasping tool suggest that arranging the objects optimally reduces the expected retrieval cost by 60--80\% in fully-observed configurations and reduces the expected search cost by 50--70\% while increasing the search success rate by up to 2x in partially-observed configurations. Supplementary material is available at \url{https://sites.google.com/berkeley.edu/osa}.
\end{abstract}

%% file: includes/1-introduction.tex
\section{Introduction}\label{sec:intro}
The ability for robots to retrieve target objects quickly from cluttered shelf environments has wide application in automation.  When shelves contain heterogeneous objects, the arrangement of objects plays a significant role in retrieval time.  For example, if frequent targets are in the front instead of the back, the robot can spend less time moving other objects out of the way to retrieve the target.
%

Prior work on robot manipulation of objects in shelf-like environments has explored task and motion planning for object rearrangement~\cite{krontiris2015dealing,nam2020fast,wang2021uniform} and the placement of relocated objects during the rearrangement process~\cite{cheong2020relocate}. 
This paper considers \emph{what makes a good arrangement}---specifically, how to arrange objects on a shelf to reduce robot retrieval time. Shelves often contain objects with varying  retrieval frequencies and movement costs. For example, in a refrigerator, objects required frequently such as milk should be near the front for easy retrieval. Olive jars and difficult-to-move objects should be near the back to avoid unnecessarily blocking the retrieval of other objects. In this work, we model the retrieval probability and movement cost of each object and formulate the problem of Optimal Shelf Arrangement (OSA): find an arrangement of objects on a shelf that minimizes the expected cost of retrieval.

\begin{figure}
    \centering
    \includegraphics[width=\linewidth]{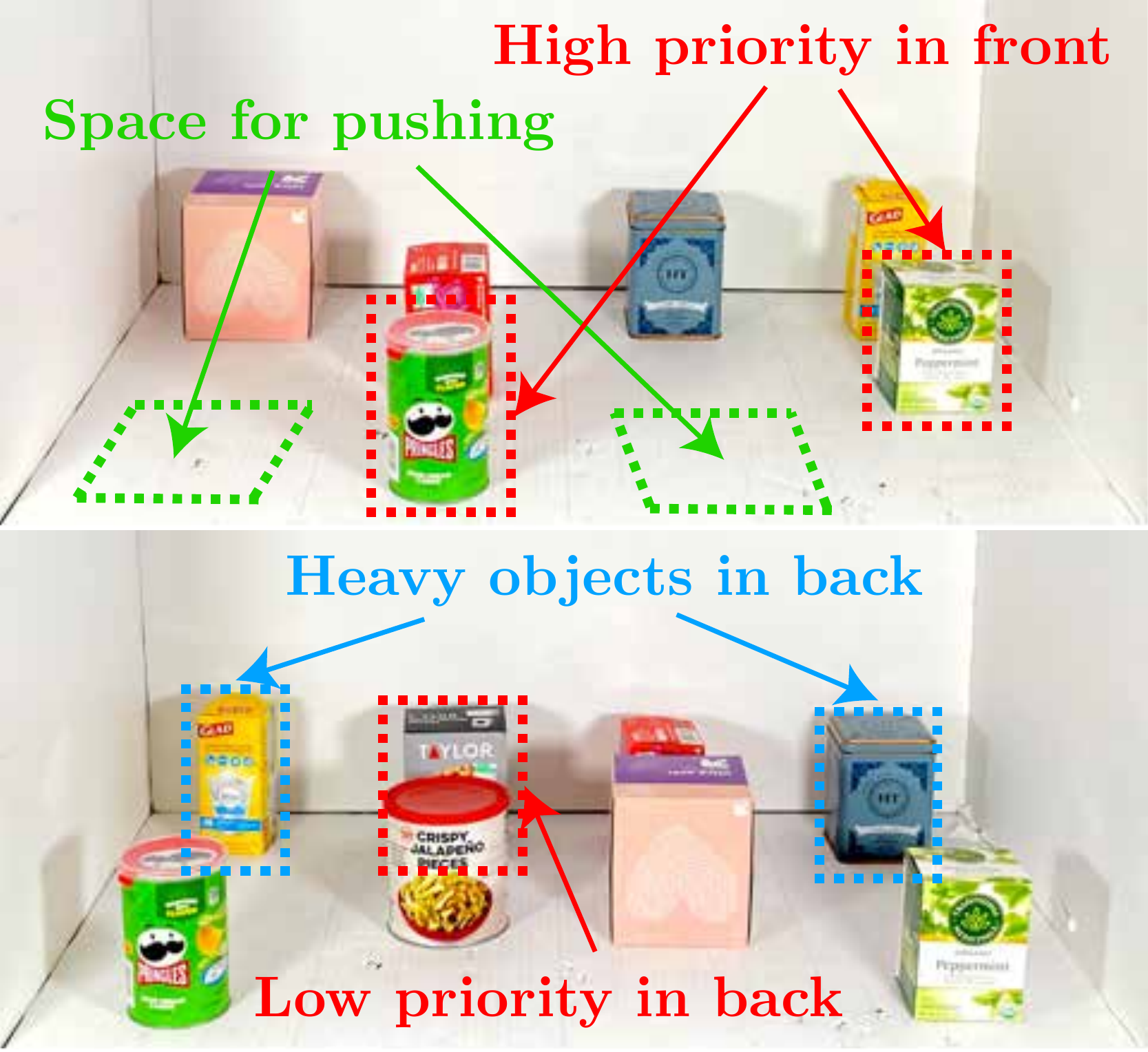}
    \caption{Two optimized shelf arrangements found by \algabbr for $n=6$ and $n=8$ objects in a discretized $4\times 3$ shelf. The object cost is proportional to its weight, with heavier objects having higher cost. Green objects have highest priorities, followed by red and yellow objects, with the blue and gray objects having the lowest priority. \algabbr trades off priority and moving costs while leaving space between objects to allow pushing actions.}
    \vspace{-10pt}
    \label{fig:bluction_splash}
\end{figure}

We consider a robot with a combined pushing blade and suction grasping end-effector, or \emph{bluction} tool~\cite{huang2022mechanical}, that can both push the objects left or right and perform pick-and-place. 
We assign costs both to each object and to the action type (i.e., the cost is action- and object-specific but path-independent) as in Han et al.~\cite{han2018efficient}, where the interpretation of the cost can be the time and effort it takes to relocate each object. The problem is challenging as it requires solving a nested optimization problem: for any arrangement, finding and computing the cost of an optimal action sequence that clears a path to retrieve each target object involves a state-space search itself (see Section~\ref{action} and~\ref{sec:sim_exp}). We bypass this nested optimization and propose \algabbr, a \emph{mixed-integer program} (MIP) that finds a near-optimal solution to the OSA problem. It discretizes the shelf and optimizes the assignment of object locations by minimizing an upper bound on the expected retrieval cost. We characterize situations where the MIP finds an exact optimal solution, and provide bounds on its suboptimality in general cases. We also give a necessary and sufficient shelf-density condition for which there exists an arrangement such that any object can be retrieved without removing objects from the shelf. We also perform simulation experiments exploring how \algabbr can change the expected retrieval cost when compared to baselines of Random and Priority-Greedy arrangements. 

While the goal of OSA is to optimize object retrieval when the configurations of all objects are known, we conjecture that an optimally-arranged shelf can also speed up object search when their positions are unknown---known as the mechanical search problem~\cite{huang2020mechanical}. We evaluate \algabbr arrangements by comparing the performance of a lateral-access mechanical search policy, SLAX-RAY~\cite{huang2022mechanical}, in both simulated and physical experiments.


This paper makes the following contributions:
\begin{enumerate}
    \item A formulation of the Optimal Shelf Arrangement (OSA) problem. Given an access frequency and object and action cost, find an arrangement of the objects on the shelf that minimizes expected retrieval cost.
    \item A mixed-integer program (MIP) finding an optimal solution to the OSA problem under certain conditions and a suboptimality gap analysis for general cost settings.
    \item Simulated experiments with 225 object sets comparing 3 arrangement policies 
    demonstrating a 60 -- 80\,\% reduction in expected one-time retrieval cost and cumulative retrieval cost (full observability case).
    \item Simulated and physical experiments comparing 3 arrangement policies 
    suggesting a 2x increase in search success on dense shelves and up to 70\,\% reduction in the expected search cost (partial observability case).
\end{enumerate}

%% file: includes/2-related-work.tex
\section{Related Work}
\subsection{Storage Assignment Policies}
Warehouse layout design and storage assignment policies in automated storage and retrieval systems (AS/RS) are well-studied~\cite{karasek2013overview}. {\"O}n{\"u}t et al.~\cite{onut2008particle} study the optimal warehouse layout design to minimize the total travel distances of the storage rack weighted by the unit handling cost of each item type and the items' yearly throughput. 
Hausman et al.~\cite{hausman1976optimal} compare random, turnover-based, and class-based turnover policies for storage assignment.
In particular, turnover-based assignment minimizes the expected one-way travel time \cite{hausman1976optimal,thonemann1998note} and has been extended to multiple command modes and performance measures~\cite{johnson1996stochastic}. Unlike AS/RS systems in which any pallet can be transported by a crane along a rectangular route, objects in a lateral shelf environment can only be accessed when all the obstacles in front of them are cleared. Thus, computing the total expected retrieval cost requires a rearrangement plan to be computed for each target object. 

\subsection{Moving-Block Problems} 
Motion planning in grid-based storage systems is challenging. One class of widely-studied problems is single-robot motion planning with movable obstacles, represented by the moving-block problems, including the games Sokuban, Pukoban, and Atomix, where a movable block called ``the man'' tries to move $k$ other movable blocks placed on a grid-square maze into specified target locations through push and/or pull actions. Pereira et al.~\cite{pereira2017solving} showed many moving-block problems are NP-hard or PSPACE-complete. Another class is remote motion planning to move a robot from its initial position to a goal position without colliding with obstacles, including the PSPACE-complete game Rush Hour~\cite{flake2002rush} and NP-Hard graph motion planning with one robot (GMP1R) problem~\cite{papadimitriou1994motion}. Other related problems include the single-item retrieval planning problem with multiple escorts (SRPME), which seeks to find a multi-robot/item motion plan that brings the requested item to one of the output locations with the minimum total number of item-moves~\cite{yalcin2018multi}, and puzzle-based storage systems~\cite{gue2007puzzle,kota2015retrieval,gue2013gridstore,mirzaei2017modelling}, where finding the fewest-move solution for a given order is NP-hard~\cite{ratner1986finding}. Unlike grid-based storage systems, where all items are routed through conveyors so they can move as long as there is space next to them, in the shelf environment, an item can only be moved when it is accessible by the robot.

\subsection{Memory Allocation}
Finding an optimal arrangement of objects on a shelf for easy retrieval has a similar goal to data and memory allocation for chip multiprocessors, and more specifically, the problem of task scheduling of applications on multiple processors to minimize energy consumption or execution time \cite{salamy2012effective}. Integer programming is a commonly used technique 
\cite{niemann1996hardware,kuang2005partitioning,avissar2002optimal,ozturk2006integer}. 
The \algabbr proposed in this work shares similarity with those integer program formulations such as Ozturk et al.~\cite{ozturk2006integer} where the data access frequencies are considered for memory partitioning when minimizing the overall energy consumption, but the constraints of \algabbr focus on clearing obstacles for retrieval instead of partitioning memory components.

\subsection{Object Rearrangement and Optimal Arrangement}
There is a rich literature on manipulation planning among movable obstacles (MAMO) \cite{stilman2007manipulation}. MAMO is a generalization of the navigation among movable obstacles (NAMO) problem, which is NP-hard \cite{stilman2005navigation}. Prior work~\cite{krontiris2015dealing,han2017high} study the task and motion planning of rearranging multiple objects from one configuration into another using pick-and-place actions on a tabletop. Shome et al.~\cite{shome2018fast,shome2020synchronized} study the multi-arm extension of the problem and formulate mixed integer programs (MIPs) using a graph representation. Cheong et al.~\cite{cheong2020relocate} focus on the subproblem of where to place relocated objects for object rearrangement. 
Wang et al.~\cite{wang2021uniform} develop a dynamic programming algorithm for uniform cylindrical-shaped objects that minimizes the number of object transfers to rearrange the objects into a goal configuration while avoiding object-object collision. 
Researchers~\cite{cosgun2011push,havur2014geometric,dabbour2019placement} have also studied the problem of placing objects onto a cluttered surface where existing objects need to first be rearranged to leave space. 

Nam et al.~\cite{nam2020fast} develop a task planner to relocate obstacles using lateral pick-and-place actions until a target object becomes reachable in a shelf environment. However, they assume all objects are cylinders graspable from any angle relative to the normal of the shelf back plane. In contrast, we assume the objects can only be accessed along the normal of the shelf back plane, and focus on finding an arrangement optimized for object retrieval instead of the rearrangement motion planning problem with a given target arrangement.


\subsection{Mechanical Search on Shelves}
Mechanical search~\cite{danielczuk2019mechanical} is the problem of locating and extracting an occluded target object when the arrangement of the objects is partially observable. Gupta et al.~\cite{gupta2013interactive} interactively explore a cluttered environment until the state of every voxel is known. Dogar et al.~\cite{dogar2014object} propose greedy and $A^*$ search algorithms for object search in a cluttered shelf. Lin et al.~\cite{lin2015planning} extend the problem setting to restrict removal of objects from the shelf and additionally consider pushing actions for large ungraspable objects. 
Recently, Huang et al.~\cite{huang2020mechanical,huang2022mechanical} proposed a SLAX-RAY search policy that maximizes the reduction in support area of a target object occupancy distribution using a ``bluction tool'' that combines a pushing blade and a suction cup gripper. In this work, we demonstrate that an optimal shelf arrangement can both increase the search success rate and reduce cost for SLAX-RAY in partially observable simulated and physical mechanical search settings.

%% file: includes/3-problem-statement.tex
\section{OSA Problem Statement} \label{action}
We consider the Optimal Shelf Arrangement (OSA) problem on a rectangular shelf discretized into a $m_x \times m_y$ grid with the coordinate frame shown in Figure~\ref{fig:misselection_of_mip}. Given $n$ objects of similar size, a distribution of object retrieval frequencies, 
and the cost of moving each object, the goal is to find an arrangement that minimizes the expected cost of retrieval. Each object $l$ occupies one cell $(i, j)$, where $i \in \lbrace 1, \, \ldots, \, m_x\rbrace, j \in \lbrace 1, \, \ldots, \, m_y\rbrace$. We call an object \textit{accessible} if there are no obstacles in front of it.

As in mechanical search~\cite{huang2022mechanical}, we consider two types of actions: pushing with a planar ``blade'' ($\mathcal{A}_p$) and suction pick-and-place ($\mathcal{A}_s$). At each timestep $t$, the robot performs an action $\mathbf{a}_t \in \mathcal{A}$ and incurs a cost $c : \mathcal{A} \rightarrow \R_{\geq 0}$, 
where $\mathcal{A} = \mathcal{A}_p \cup \mathcal{A}_s$.

\textbf{1) Pushing actions.}
Pushing actions in $\mathcal{A}_p$, parameterized as $\mathbf{a}_t = (i, j, d)$, where $d \in \{-1, 1\}$, pushes an object located at $(i, j)$ along the $x$-axis of the shelf frame to $(i+d, j)$. This action is only possible when the object is accessible and pushed to an empty space. We assume there is enough space between objects in neighboring cells to insert the blade and execute the pushing action.

\textbf{2) Suction actions. }
Suction actions in $\mathcal{A}_s$, parameterized as $\mathbf{a}_t = (i, j, i', j')$, start with the robot forming a seal between the suction cup and object, followed by 3 linear motions:
(1) lifting,
(2) translating the object along the $x$- and $y$-axes to column $i'$, row $j'$, and
(3) placing the object at its final position $(i', j')$. This action is only possible if the object to be suctioned is accessible, the target position is empty, and there are no objects in front of the target position.

\textbf{3) Removal actions. }
When there are no available placement positions on the shelf, a suction action $\mathbf{a}_t = (i, j, -1,$ $-1) \in \mathcal{A}_s$ may place the object outside the shelf. 
The robot incurs a penalty for removing an object from the shelf, as it results in more planning and execution time to move the robot base to access buffer spaces outside the shelf~\cite{cheong2020relocate}.

OSA requires the following parameters:
\begin{itemize}
  \item $p_l$: a probability distribution for requesting object $l$;
  \item $c_{lp}$: the cost of moving object $l$ with a pushing action;
  \item $c_{ls}$: the cost of moving object $l$ with a suction action;
  \item $c_r$: the penalty for removing an object from the shelf.
\end{itemize}

Let $\Delta c_l=c_{ls} - c_{lp}$, and $\Delta c_l \geq 0 ~\forall l$. We assume $\Delta c_l \geq 0$ because suction actions are more prone to failures such as suction seal loss than pushing actions and because prior work~\cite{huang2022mechanical} empirically found that suction pick-and-place actions take approximately 1.3 times longer than pushing actions. Thus, we set $c_{ls} = 1.3 c_{lp}$ in our simulated and physical experiments, but the OSA problem and the proposed \algabbr (Section \ref{sec:discrete_mip}) allow general $c_{lp}$ and $c_{ls}$ values.

Given an arrangement $\mathcal{S}: [1 .. n] \rightarrow [1 .. m_x] \times [1 .. m_y]$, we let $\mathbf{A}_l^\mathcal{S} = \{\mathbf{a}_1^{\mathcal{S}}(l),  ..., \mathbf{a}_{t_l}^{\mathcal{S}}(l)\}$ denote an optimal sequence of actions to retrieve target object $l$ such that all obstacles in front of $l$ have been cleared and the total cost of the action sequence $\mathrm{cost}(l^\mathcal{S}) = \sum_{i=1}^{t_l} c(\mathbf{a}_{i}^{\mathcal{S}}(l))$ is minimized. The final retrieval action is not included in $\mathbf{A}_l^\mathcal{S}$ since its cost is the same for all configurations. The cost of an arrangement $\mathcal{S}$ is the expected cost of retrieval among all objects: $C(\mathcal{S}) = \mathbb{E}_{l \sim p_l}[\text{cost}(l^\mathcal{S})]$. The goal is to find an arrangement $\mathcal{S}^* = \argmin_\mathcal{S} C(\mathcal{S})$ that minimizes the expected cost of retrieval.

%% file: includes/4-shelf-density.tex
\section{Shelf Density Analysis}
\label{sec:density}
We define the density of a shelf as $\rho := \frac{n}{m_xm_y}$. A shelf is \emph{dense} if $\rho > 1 - \frac{1}{m_x} + \frac{1}{m_xm_y}$. In Theorem~\ref{thm:density}, we show that when the density exceeds this threshold, removal is unavoidable for retrieving objects at the back of the shelf. For a square shelf ($m_x = m_y$), this threshold is at least $\min_{m>0} (1 - \frac{1}{m} + \frac{1}{m^2}) = 0.75$. An arrangement $\mathcal{S}$ is \textit{hollow} if there exists an empty cell behind an occupied cell, and such empty cell is called a \textit{cavity}. The \textit{consolidation} process pushes all objects as far back as possible to remove all cavities from a hollow arrangement, and we denote the resulting configuration $\mathcal{S}_\mathrm{consol}$.

\begin{theorem}\label{thm:density}
There exists an arrangement that does not require removal of objects to retrieve any target if and only if the shelf is not dense.
\end{theorem}

\begin{proof}
First, we observe that for any retrieval target, a hollow arrangement will never reduce the number of removals required compared to its consolidated version. This is because if $k$ objects are in front of the target and need to be relocated, no removal is required if and only if the total number of empty cells in other columns is at least $k$; if some empty cells are cavities, then one may need to first consolidate the arrangement so that these empty cells are in the front of the shelf to be available to place relocated obstacles. Thus, for any $m_x \times m_y$ cell, to calculate the maximum number of objects that can fit onto the shelf without requiring removal for retrieving any objects, we only need to consider consolidated arrangements. 

Suppose the back $m_y-k$ rows are full, and there are $r$ objects in the first $k$ rows. We only need to make sure the backmost object can be successfully retrieved without requiring any removal in each column. Consider any column and assume it has $m_y - k + u$ objects. In order to retrieve the backmost object, we need to relocate $m_y - k + u - 1$ objects into the other columns, which has space equals $k (m_x-1) - (r - u)$ cells. So no removal is required if and only if 
$$k (m_x-1) - (r - u) \geq m_y - k + u - 1,$$
which simplifies to $k \geq \frac{m_y + r -1}{m_x}$.
By plugging $k$ into density $\rho = \frac{\left(m_y - k\right)m_x+r}{m_x m_y}$, we get $\rho \leq \frac{\left(m_y - \frac{m_y + r -1}{m_x}\right) m_x + r}{m_x m_y} = 1 - \frac{1}{m_x} + \frac{1}{m_xm_y}$. Any denser shelf would not have an arrangement that avoids removal for all target objects.
\end{proof}

%% file: includes/5-mip-formulation.tex
\section{MIP Formulation} \label{sec:discrete_mip}
We propose a mixed-integer program (MIP) to solve OSA and analyze its optimality in Section~\ref{sec:optimality_analysis}. By assumption, to retrieve a target object, all obstacles in front of it must be relocated. For each obstacle to be moved, all objects in front of it must be cleared by pushing or suction actions. As we assume $c_{ls} \geq c_{lp} ~\forall l$, pushing actions are preferred over suction. Given these observations, we formulate the MIP as:\\
\noindent
\textbf{Indices:}
\begin{itemize}[noitemsep,topsep=0pt,parsep=0pt,partopsep=0pt,label={},leftmargin=*]
  \item[]$l \in \lbrace 1, \, \dots, \, n\rbrace$: object index;
  \item[] $i \in \lbrace 1, \, \dots, \, m_x \rbrace$: x position of the grid;
  \item[] $j \in \lbrace 1, \, \dots, \, m_y \rbrace$: y position of the grid.
\end{itemize}
\noindent
\textbf{Exogenous parameters:}
\begin{itemize}[noitemsep,topsep=0pt,parsep=0pt,partopsep=0pt, label={},leftmargin=*]
    \item $p_l, c_{lp}, c_{ls}, c_r$: retrieval frequency and moving costs for object $l$.
\end{itemize}
\noindent
\textbf{Decision variables:}
\begin{itemize}[noitemsep,topsep=0pt,parsep=0pt,partopsep=0pt] 
  \item[$x_{l i j}$] $\in \lbrace 0, 1\rbrace$: indicates if object $l$ is at $(i, j)$;
  \item[$y_{l i j}$] $\geq 0$: cost of retrieving object $l$ from $(i, j)$, without including removal penalties;
  \item[$a_{i j}$] $\in\{0,1\}$:  indicates if cell $(i, j)$ is occupied;
  \item[$e_{i j}$] $\in \lbrace 0, \, \ldots, \, m_y\rbrace$: the number of objects in front of and including cell $(i, j)$;
  \item[$f_{i j}$] $\in\lbrace 0,1\rbrace$: indicates if all cells in front of and including $(i, j)$ are empty;
  \item[$d_{i}$] $\in \lbrace 0, \, \ldots, \, m_y\rbrace$: the number of empty cells at the front of column $i$;
  \item[$\delta_{i j}$] $\in\lbrace 0,1\rbrace$: indicates if either $(i-1, j)$ or $(i+1, j)$ are open (\ie pushes are available);
  \item[$\delta^p_{l i j},$] $\delta^s_{l i j} \in\lbrace 0,1\rbrace$: indicates if object $l$ is at $(i, j)$ and can be pushed or can only be suctioned, respectively;
  \item[$\tilde\delta_{i j}$] $\in\lbrace 0,1\rbrace$: indicates if there is an object at $(i, j)$ that can be pushed behind another object;
  \item[$b_{l i j}$] $\in \lbrace 0, \, \ldots, \, m_y-1\rbrace$: the number of objects that must be removed from the shelf to retrieve object $l$ from $(i, j)$.
\end{itemize}
\noindent
\textbf{Objective:}
\begin{align}
    \text{Minimize } &\sum_{l=1}^{n} \sum_{i=1}^{m_x} \sum_{j=1}^{m_y} p_{l}(y_{l i j}+c_r b_{l i j}) \label{mip_obj}
\end{align}
The objective~\eqref{mip_obj} minimizes the retrieval cost $y_{l i j}$ for each object $l$ (which depends on $c_{l p}$ and $c_{l s}$) plus any removal penalty incurred, weighted by object $l$'s retrieval frequency $p_l$. We show in Section~\ref{sec:optimality_analysis} that this quantity is an upper bound on the true expected cost $C(\mathcal{S})$. \\
\noindent
\textbf{Constraints:}

\begin{align}
    &\sum_{i=1}^{m_x} \sum_{j=1}^{m_y} x_{l i j} =1, &&\forall l \label{mip_object}
\end{align}
\begin{align}
    & a_{i j} = \sum_{l=1}^{n} x_{l i j} \leq 1  &&\forall i, j \label{mip_position} \\
    & \delta_{i j} = \neg \left( a_{i-1, j} \land a_{i+1, j}\right), &&\forall i, j \nonumber \\
    &\delta^p_{l i j} = x_{l i j} \land \delta_{i j}, \quad \delta^s_{l i j} = x_{l i j}, \land \neg \delta_{i j} &&\forall l, i, j \nonumber \\
    &e_{i j} = \sum_{j^{\prime}=1}^{j} a_{i j'},\quad f_{i j} = \mathds{1}\left(e_{ij} = 0\right), &&\forall i, j \nonumber \\
    &d_i = \sum_{j=1}^{m_y} f_{i j}, &&\forall i \nonumber \\
    y_{l i j} =& \resizebox{0.69\columnwidth}{!}{%
                    $\begin{cases} 
                    \resizebox{0.5\columnwidth}{!}{%
                    $\sum_{l'\neq l} \sum_{j^{\prime}=1}^{j-1} [c_{l' p} \delta^p_{l' i j'} + c_{l' s} \delta^s_{l' i j'}]$%
                    }  & \text{if } x_{l i j} = 1,\\
                    0 & \text{if } x_{l i j} = 0,
                \end{cases}$%
                    }  &&\forall l, i, j \nonumber \\
    \tilde \delta_{i j} =& \resizebox{0.69\columnwidth}{!}{%
                    $a_{ij} \land \neg \left( \left( a_{i-1, j} \lor f_{i-1, j} \right) \land \left( a_{i+1, j} \lor f_{i+1, j}\right)\right)$%
                    } &&\forall i, j \nonumber\\
    b_{l i j} =& \resizebox{0.69\columnwidth}{!}{%
                    $\begin{cases} 
                    \resizebox{0.8\columnwidth}{!}{%
                    $\max{\left(e_{i j} - x_{l i j} - \sum_{j^{\prime}=1}^{j-1} \tilde \delta_{i j'} - \sum_{i^{\prime} \neq i} d_{i'}, 0\right)}$%
                    }                     & \text{if } x_{l i j} = 1,\\
                    0 & \text{if } x_{l i j} = 0,\\
                \end{cases}$%
                    }  &&\forall l, i, j \nonumber\\
    &x_{l i j}, f_{i j}, \delta_{i j}, \delta^p_{l i j}, \delta^s_{l i j}, \tilde \delta_{i j}, a_{i j} \in\lbrace0,1\rbrace, \quad &&\forall l, i, j, \nonumber\\
    &y_{l i j}  \in \R_{\geq 0},  b_{l i j}, d_i, e_{i j} \in \Z_{\geq 0}, \quad &&\forall l, i, j. \nonumber
\end{align}
    
Constraints \eqref{mip_object} and \eqref{mip_position} ensure that each object has exactly 1 position and that each position has at most 1 object. All other constraints formally define the decision variables described above.
All of the switch-case constraints can be linearized by introducing auxiliary variables; a fully expanded formulation is included in the Appendix. By Theorem \ref{thm:density}, if the shelf is not dense, then an arrangement that does not require objects to be removed for retrieval of any target object is possible. In this case, we can add a constraint $\sum_{l=1}^{n} \sum_{i=1}^{m_x} \sum_{j=1}^{m_y} p_{l}b_{l i j} \leq 0$ in place of a large $c_r$.

%% file: includes/6-mip-analysis.tex
\section{Analysis of \algabbr} \label{sec:optimality_analysis}

In this section, we analyze the optimality of \algabbr's solutions. Because the retrieval cost depends not only on the total number of actions but also on each object acted on, finding an optimal action sequence $\mathbf{A}_l^\mathcal{S}$ to retrieve each object requires a state-space search in the general cost settings (see Section~\ref{sec:sim_exp}). Preemptive actions that first rearrange the available shelf space, instead of only relocating obstacles in front of the target, may end up being cheaper even if they require more action steps. However, finding an optimal arrangement $\mathcal{S}^* = \argmin_\mathcal{S} \mathbb{E}_{l \sim p_l}[\text{cost}(l^\mathcal{S})]$ where each inner cost computation requires solving a tree-search subproblem for every target object can be intractable. 

\algabbr computes an upper bound for the optimal retrieval cost for each object and minimizes that upper bound. 
Formally, we call an action sequence $\mathbf{A}_l^\mathcal{S}$ \textit{non-preemptive} if it only involves relocating obstacles in front of the target object $l$, and \textit{preemptive} otherwise. Let $\hat{\text{cost}}(l^\mathcal{S})$ be the optimal cost of retrieving target object $l$ and $\hat C(\mathcal{S}) = \mathbb{E}_{l \sim p_l}[\hat{\text{cost}}(l^\mathcal{S})]$ be the corresponding expected retrieval cost when only non-preemptive actions are allowed. Then, as \algabbr only considers non-preemptive actions for tractability, it optimizes for $\hat C(\mathcal{S})$ instead of $C(\mathcal{S})$, with $\hat C(\mathcal{S}) \geq C(\mathcal{S}) ~\forall \mathcal{S}$, and may obtain a suboptimal solution $\mathcal{S}_{MIP}$ when $\mathcal{\hat C(S}_{MIP}) < \mathcal{\hat C(S^*})$ but $\mathcal{C(S}_{MIP}) > \mathcal{C(S^*)}$, where $\mathcal{S^*}$ is an optimal arrangement.


\begin{figure}
    \centering
    \vspace{4pt}
    \includegraphics[width=0.85\linewidth]{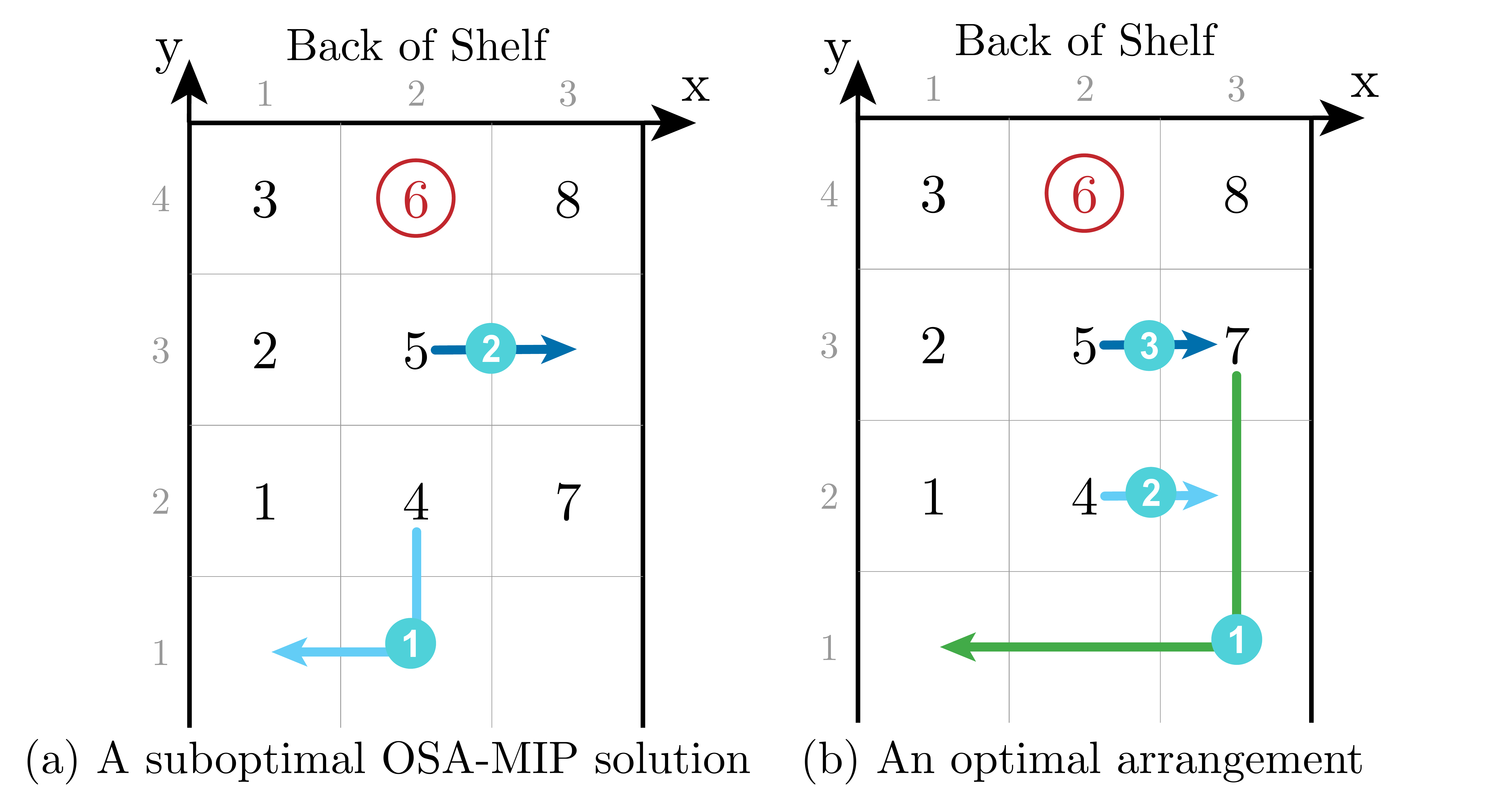}
    \caption{(a) An example where \algabbr finds a suboptimal arrangement for retrieving object 6 if $\Delta c_4 \leq c_{7s} \leq \Delta c_5$ and $(p_5+p_6) \Delta c_4 > p_6 c_{7s}$ and (b) an optimal arrangement resulting in a lower cost retrieval despite more actions. In (b), the optimal action sequence is to suction object 7 so that object 5 can then be pushed. As \algabbr does not consider preemptive actions, it instead leaves a cavity at (3,3) to push object 5.}
    \vspace{-4pt}
    \label{fig:misselection_of_mip}
\end{figure}

\begin{figure}
    \centering
    \includegraphics[width=0.9\linewidth]{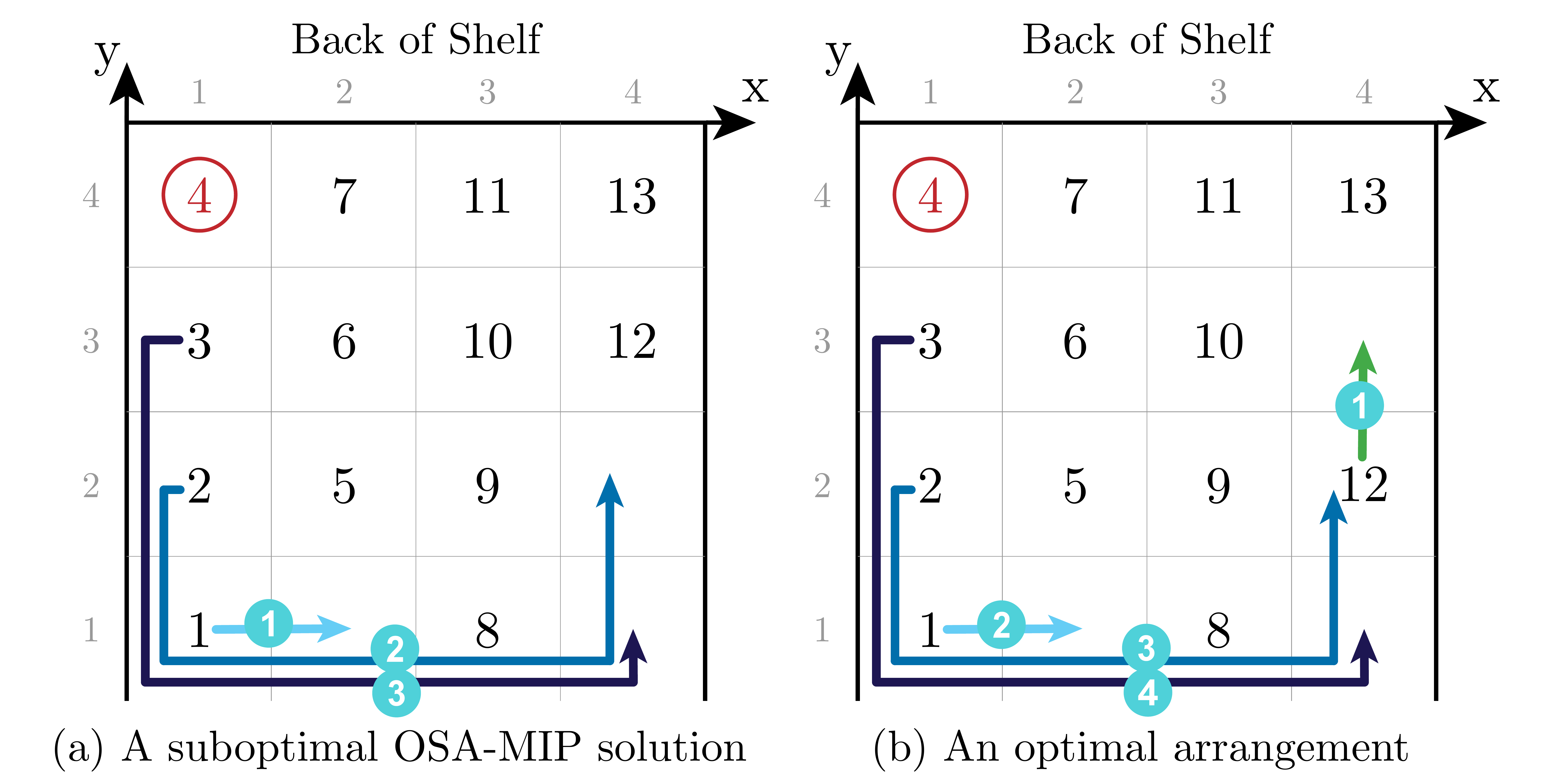}
    \caption{Another example of (a) a suboptimal \algabbr solution and (b) an optimal arrangement. To retrieve object 4, the optimal action sequence is to suction object 12 backward to make space for relocating objects 1-3. As \algabbr does not consider preemptive actions, it will place object 12 at (4,3) to avoid removing objects if $c_r$ is large. However, this arrangement is suboptimal if $p_4+p_7$ and $p_{10}$ are small compared to $p_{11}$ and $\Delta c_{10} > \Delta c_9$.}
    \vspace{-12pt}
    \label{fig:suboptimality_of_mip}
\end{figure}

Figure~\ref{fig:misselection_of_mip} illustrates an example where \algabbr may find a suboptimal solution. In particular, if object 7 has a small suction cost, it may be cheaper to retrieve object 6 in arrangement (\ref{fig:misselection_of_mip}b) than (\ref{fig:misselection_of_mip}a) even though it takes one more action, as illustrated by the arrows in Figure~\ref{fig:misselection_of_mip}. However, as \algabbr does not consider preemptive actions (the green arrow) and will use a pushing action on object 4 and a suction action on object 5, it overestimates the retrieval costs for arrangement (\ref{fig:misselection_of_mip}b). It can be shown that if $p_6 (\Delta c_5 - \Delta c_4) > p_5 \Delta c_4 > p_6 (c_{7s} - \Delta c_4)$, \algabbr will suboptimally find (\ref{fig:misselection_of_mip}a) even though (\ref{fig:misselection_of_mip}b) is optimal.


\begin{figure*}[t!]
    \centering
    \includegraphics[width=\linewidth]{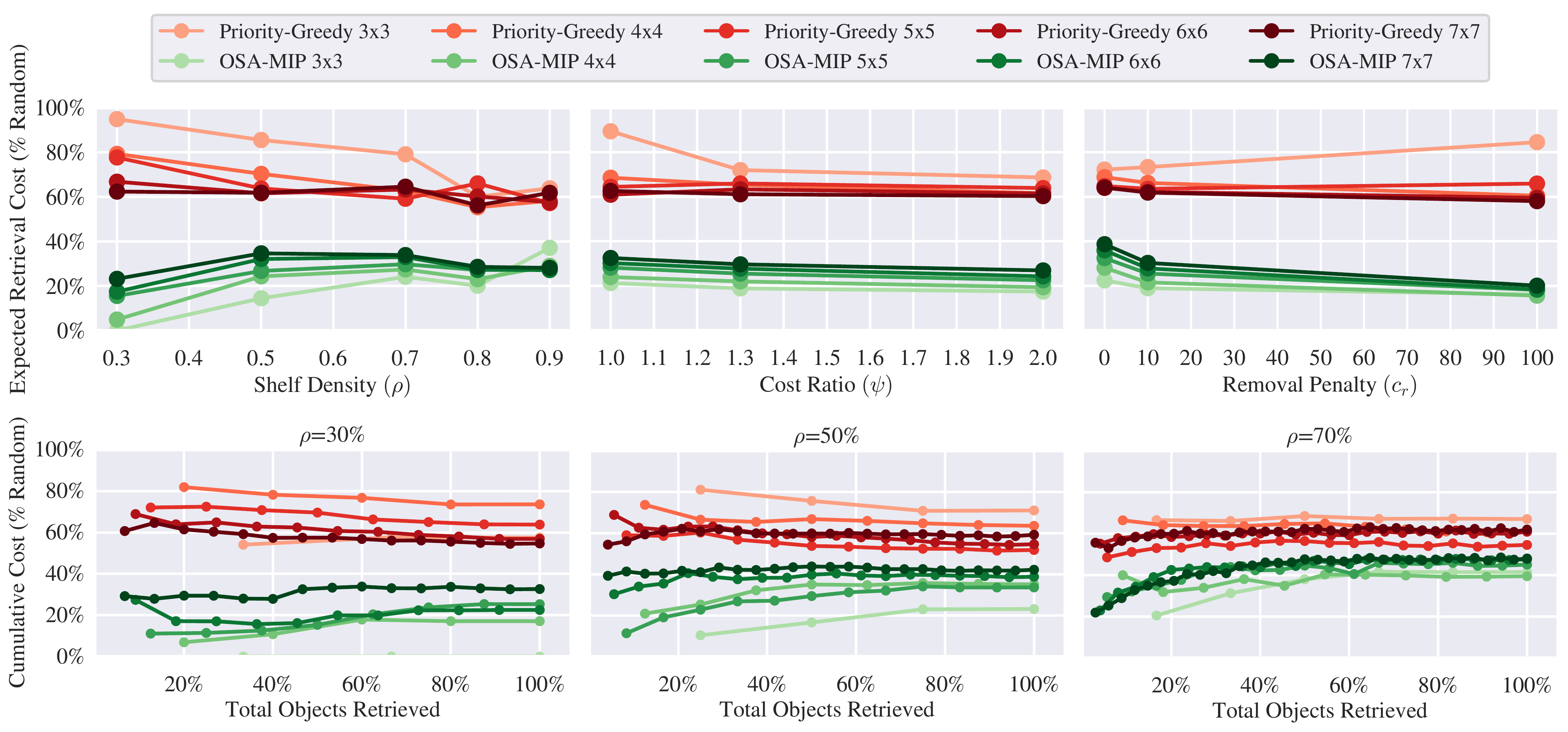}
    \caption{(Top row) Effect of shelf density $\rho$ (left), cost ratio $\psi$ (middle), and removal penalty $c_r$ (right) on the cost of \algabbr and Priority-Greedy arrangements as a percentage of random arrangement costs across shelf sizes. (Bottom row) Cumulative retrieval cost for \algabbr and Priority-Greedy arrangements as a percentage of random arrangement costs for the sequential retrieval task across different shelf densities. \algabbr outperfoms baselines despite optimizing only for a single retrieval; however, for denser shelves or as more objects are retrieved, the relative benefit of an \algabbr arrangement decreases.}
    \vspace{-12pt}
    \label{fig:mip_hyperparameters}
\end{figure*}

Additionally, \algabbr may also produce suboptimal consolidated arrangements to avoid removal when in fact hollow arrangements are optimal, as shown in Figure~\ref{fig:suboptimality_of_mip}. The arrangement in Figure~\ref{fig:suboptimality_of_mip}b does not require removal, but to retrieve object 4 or 7, object 12 must be pushed backward to make space for relocating objects 1-3 or 5-6, respectively. As \algabbr does not consider this backward push, it will think arrangement (\ref{fig:suboptimality_of_mip}b) requires removal and overestimate its cost, and it will put object 12 at (4,3) position to avoid removing objects, as shown in Figure~\ref{fig:suboptimality_of_mip}a even though (\ref{fig:suboptimality_of_mip}b) can be optimal as it enables pushing object 10. The Appendix contains a full cost analysis of Figures~\ref{fig:misselection_of_mip} and \ref{fig:suboptimality_of_mip}. 


However, in cases where the cost saving for pushing compared to suctioning is the same for all objects and preemptive actions are not beneficial, \algabbr will be exact, as the following theorem states.

\begin{theorem}\label{correctness_of_mip}
If $c_{ls} - c_{lp} \equiv \Delta c \leq \min_l c_{lp} \text{ }\forall l$, then \algabbr finds an optimal solution.
\end{theorem}

\begin{proof}
When $c_{ls} - c_{lp} \equiv \Delta c \text{ }\forall l$ is a constant, cavity locations are irrelevant (\eg Figures \ref{fig:misselection_of_mip}a and \ref{fig:misselection_of_mip}b are equivalent). If $\Delta c \leq \min_l c_{lp}$, preemptive actions to rearrange the space are not beneficial because each preemptive action creates at most one space and 1 suction action is cheaper than 2 pushing actions. Therefore, \algabbr's cost is exact and it will produce an optimal solution.
\end{proof}

\begin{theorem}\label{bound_of_mip_suboptimality_gap}
Let $\mathcal{S^*}$ be an optimal arrangement and $\mathcal{S}_{MIP}$ be an arrangement found by \algabbr. Then
\begin{align}
\mathcal{C(S^*)} &\leq \mathcal{C(S}_{MIP}) \\&\leq \min \left\{\mathcal{\hat C(S}^*_{consol}), \mathcal{\hat C(S^*)} \right\} \\
 & \leq \min \left\{ k\mathcal{C(S^*)}, \mathcal{\hat C(S^*)} \right\},
\end{align}
where $k = \max_l\{\frac{c_{ls}}{c_{lp}}\}$. 
\end{theorem}

\begin{proof}
We always have $\mathcal{C(S^*)} \leq \mathcal{C(S}_{MIP}) \leq \mathcal{\hat C(S}_{MIP}) \leq \min \left\{\mathcal{\hat C(S}^*_{consol}), \mathcal{\hat C(S^*)} \right\}$, where the last inequality follows from \algabbr optimizing $\mathcal{\hat C}(.)$.

To show $\mathcal{\hat C(S}^*_{consol}) \leq k \mathcal{C(S^*)}$, let $\mathcal{\bar{C}(S)}$ be the expected retrieval cost for arrangement $\mathcal{S}$ if we increase $c_{lp}$ to be $c_{ls} ~\forall l$ and keep $c_r$ the same. As argued in Theorem \ref{correctness_of_mip}, when $c_{lp} = c_{ls} ~\forall l$, preeemptive actions have no benefit. Thus, $\mathcal{\hat C(S}^*_{consol}) \leq \mathcal{\bar{C}(S}^*_{consol}) = \mathcal{\bar{C}(S}^*) \leq k \mathcal{C(S^*)}$, where the equality is because when pushing and suction costs are equal, the cost of retrieving each object $l$, $\text{cost}(l^\mathcal{S})$, is simply the cost of relocating (or removing if necessary) all obstacles in front of $l$, which is equal to $\text{cost}(l^{\mathcal{S}_{consol}})$ for all $\mathcal{S}$.
\end{proof}




%% file: includes/7-sim-experiments.tex
\section{\algabbr Planning Experiments}\label{sec:sim_exp}

We conduct simulated experiments to evaluate the benefit of \algabbr arrangements, as compared to random arrangements or arrangements generated by a priority-greedy algorithm while varying shelf size ($m_x, m_y$), density $\rho$, ratio of suction cost to pushing cost $\psi$, and removal penalty $c_r$. Additionally, we evaluate \algabbr's ability to generate arrangements for the sequential retrieval task. We use a Gurobi solver~\cite{gurobi} for all of our experiments. For an analysis of \algabbr runtime, please refer to the Appendix.

We use 5 shelf sizes $(m_x, m_y) \in $ $\lbrace (3,3),$ $(4,4),$ $(5,5),$ $(6,6),$ $(7,7)\rbrace$, 5 density values $\rho \in \{0.3, 0.5, 0.7, 0.8, 0.9\}$, 3 cost ratios $\psi \in \lbrace 1.0, 1.3, 2.0 \rbrace$, and 3 removal penalties $c_r \in \lbrace 0, 10, 100\rbrace$. We sample 3 pushing costs $c_{lp}$ for each object from a discrete uniform distribution on $\{1, \, \ldots, \, 10\}$ and let retrieval frequency be inversely proportional to priority (\ie the probability of being retrieved is $n:(n-1): ... :2:1$ for $n$ objects). We compare \algabbr to two baseline arrangement algorithms on each of the $5 \times 5 \times 3 \times 3 \times 3 = 675$ object configurations: 
\begin{enumerate}
\item \textbf{Random}: objects are assigned to a position on the shelf drawn uniformly at random.
\item \textbf{Priority-Greedy}: we select $n$ positions on the shelf uniformly at random, sort the positions by their depth in the shelf, then assign the $n$ objects to the positions in the order of priority (\ie the highest priority object is assigned to the frontmost position). We use random positions instead of placing all the objects at the front or back of the shelf because the latter approach will not leave space between objects for pushing actions.
\end{enumerate}

To evaluate the true expected cost $C(\mathcal{S})$ for each arrangement, we use an A* search algorithm to find an optimal action sequence $\mathbf{A}_l^\mathcal{S}$ for retrieving each object and compute the corresponding optimal cost $\text{cost}(l^\mathcal{S})$, weighted by the probability $p_l$. Preemptive actions are considered: at each step, pushing, suction, and removal actions can be performed on the frontmost object in each column. We use the sum of pushing costs of all obstacles in front of the target object as an admissible heuristic since it is a lower bound on the retrieval cost. However, A* search becomes intractable in relatively dense ($\rho \geq 0.7)$ shelves for target objects at the back both due to the large branching factor and the large number of actions needed to retrieve the object. We switch to a non-preemptive retrieval policy that computes $\hat{\text{cost}}(l^\mathcal{S})$ when A* search does not find a solution within 1 minute.

Figure~\ref{fig:mip_hyperparameters} plots the cost of \algabbr and Priority-Greedy arrangements as a percentage of Random arrangements for each setting. Since retrieval cost increases as the shelf becomes denser and the number of objects, the suction cost, or the removal penalty increases, we normalize the average cost of \algabbr and Priority-Greedy arrangements by that of the Random arrangement. We can see that the expected retrieval cost for \algabbr arrangements are 20-40\% of the cost of Random arrangements, while Priority-Greedy arrangements have 60-80\% of the cost of Random arrangements. As the density increases and the shelf becomes larger, the relative benefit of \algabbr decreases slightly but remains nearly half the cost of the Priority-Greedy arrangements. The effect of $\psi$ and $c_r$ on the relative cost is small, although \algabbr performs relatively better when $c_r$ is large because it directly optimizes to avoid removal, unlike Priority-Greedy and Random.

We also evaluate the cumulative cost for sequential object retrieval. We iteratively sample target objects according to the retrieval distribution normalized by the current objects, execute the retrieval, and repeat until all objects have been retrieved or removed. Figure~\ref{fig:mip_hyperparameters} shows the results, which suggest that \algabbr has lower cumulative cost, especially for less-dense shelves. As more objects are retrieved and the initial arrangement is disturbed, the relative benefit of \algabbr decreases since it optimizes for a single retrieval.

%% file: includes/8-mech-search.tex
\section{\algabbr for Mechanical Search}
\input{tables/table1}
While OSA assumes full observability and optimizes for expected retrieval cost, we hypothesize that an optimal arrangement can also benefit lateral-access mechanical search, where a known target object must be searched from a partially observed shelf of unknown objects. We evaluate Random, Priority-Greedy, and \algabbr arrangements for mechanical search using SLAX-RAY~\cite{huang2022mechanical}. SLAX-RAY calculates the target occupancy distribution, which indicates the probability of each object occluding the target object, then greedily reduces the occupancy distribution support through pushing and suction actions using a bluction tool. In contrast to the action set in Section~\ref{action}, removal actions are not considered by SLAX-RAY and objects cannot be moved behind other objects, so we set $c_r$ to a large value in \algabbr to discourage removal actions. We use a shelf size of \SI{0.8x0.5x0.5}{\meter} in both simulation and physical experiments. A trial is successful when the target is retrievable.

\subsection{Simulation Experiments}
We conduct experiments in simulation using the First Order Shelf Simulator (FOSS) simulator~\cite{huang2020mechanical}. We discretize the shelf into a $5\times4$ grid with 6, 8, 10, 12, 14 or 16 cuboid objects, with each setting repeated 50 times for each arrangement. The objects are generated with random dimensions within the limit of the grid size, and are assigned retrieval probabilities $p_l$ drawn from a standard uniform and normalized. As heavier objects tend to be more difficult to move, we set the pushing cost of each object to be proportional to its volume and use a cost ratio of $\psi=1.3$. We evaluate each arrangement based on the number of steps, success rate, and retrieval cost SLAX-RAY produces for each object, weighted by $p_l$. SLAX-RAY fails if it takes $2n$ actions without revealing the target object or if there are no allowable actions (e.g., all cells in the front row are occupied). 

Table~\ref{tab:results_slaxray_sim_54} shows the results. The mean steps and cost are computed only among successful searches. ``Visible objects'' refer to the case where the target requested is already placed at the front of a column and so is directly retrievable without any steps of search. We see that \algabbr arrangements not only have the largest proportion of time when the requested target is readily retrievable, but the search success rate among hidden objects remains 99-100\% even for relatively dense shelves, and is 40-43\% higher than baselines for 16 objects. The smaller mean steps and cost for baselines in dense shelves are because SLAX-RAY can only successfully find objects in the front rows which has fewer steps, while \algabbr leaves sufficient space for SLAX-RAY to manipulate and find deeply-hidden objects at the back so its weighted average steps and cost among all successful trials are larger. 


\subsection{Physical Experiments}
We use a physical Fetch robot with a ``bluction'' tool~\cite{huang2022mechanical} and an Intel RealSense LiDAR Camera L515 for RGBD observations. We discretize the shelf into a $4\times3$ grid, select a set of 4, 6, or 8 household objects (shown in Figure~\ref{fig:bluction_splash}), and assign them a priority order and cost based on their weight. We run the same experiments for each object set as in simulation. Table~\ref{tab:results_physical_slaxray} shows the results. As the shelf is not dense, SLAX-RAY successfully finds all targets in all arrangements so there is no bias in the means due to the exclusion of failures as in Table~\ref{tab:results_slaxray_sim_54}. We see that \algabbr achieves 2x fewer steps and 3x lower cost to find hidden objects on average, suggesting that mechanical search policies can benefit from optimally arranged objects. A direct comparison with simulation, where the shelf is discretized into a $4\times3$ grid, is included in the Appendix.

\input{tables/table2}

%% file: tables/table1.tex
\newcommand{\SR}{Success Rate}
\newcommand{\SA}{Mean $\pm$ Std Steps}
\newcommand{\SC}{Mean $\pm$ Std Cost}

\newcommand{\sr}[1]{\ifthenelse{\equal{#1}{*}}{\srStar}{\srNoStar{#1}}}
\newcommand{\srStar}[1]{\textbf{\srNoStar{#1}}}
\newcommand{\srNoStar}[1]{#1\,\%}

\newcommand{\srp}[1]{\ifthenelse{\equal{#1}{*}}{\srpStar}{\srpNoStar{#1}}}
\newcommand{\srpStar}[2]{\textbf{\srpNoStar{#1}{#2}}}
\newcommand{\srpNoStar}[2]{#1 $\pm$ #2\,\%}

\newcommand{\sa}[1]{\ifthenelse{\equal{#1}{*}}{\saStar}{\saNoStar{#1}}}
\newcommand{\saStar}[2]{\textbf{\saNoStar{#1}{#2}}}
\newcommand{\saNoStar}[2]{#1 $\pm$ #2}

\newcommand{\scos}[1]{\ifthenelse{\equal{#1}{*}}{\scStar}{\scNoStar{#1}}}
\newcommand{\scStar}[2]{\textbf{\scNoStar{#1}{#2}}}
\newcommand{\scNoStar}[2]{#1 $\pm$ #2}

\newcommand{\NTT}{\% Hidden Objects}
\newcommand{\TTT}{\% Visible Objects}
\newcommand{\SNT}{\% Success Hidden}
\newcommand{\SAT}{Mean Steps Hidden}
\newcommand{\SCT}{Mean Cost Hidden}

\begin{table}[t]
\vspace{4pt}
\centering
\resizebox{\columnwidth}{!}{
\begin{tabular}{@{}ccccc@{}}\toprule
No. & Metric & Random  & Priority-Greedy & \algabbr \\

\midrule
\multirow{3}*{6} 
  & \TTT & \srp{61}{15}  & \srp{75}{12}      & \srp*{93}{7}    \\
  & \SNT & \srp*{100}{0}  & \srp*{100}{0}      & \srp*{100}{0}    \\
  & \SAT & \sa{1.3}{0.5} & \sa{1.2}{0.2} & \sa*{1.1}{0.3}\\
  & \SCT & \scos{10.8}{4.1} & \scos{10.0}{2.3} & \scos*{6.9}{2.9}\\
\cmidrule{1-5}
\multirow{3}*{8} 
  & \TTT & \srp{56}{13}  & \srp{71}{9}      & \srp*{81}{7}    \\
  & \SNT & \srp{99}{3}  & \srp*{100}{0}      & \srp*{100}{0}    \\
  & \SAT & \sa{1.5}{0.8} & \sa{1.3}{0.4} & \sa*{1.2}{0.5}\\
  & \SCT & \scos{12.5}{7.3} & \scos{11.5}{4.7} & \scos*{8.1}{2.7}\\
\cmidrule{1-5}
\multirow{3}*{10} 
  & \TTT & \srp{51}{11}  & \srp{67}{7}      & \srp*{73}{6}    \\
  & \SNT & \srp{97}{15}  & \srp{97}{15}      & \srp*{100}{0}    \\
  & \SAT & \sa{1.8}{0.7} & \sa{1.4}{0.4} & \sa*{1.1}{0.4}\\
  & \SCT & \scos{15.5}{6.4} & \scos{11.5}{4.1} & \scos*{8.7}{2.9}\\
\cmidrule{1-5}
\multirow{3}*{12} 
  & \TTT & \srp{41}{8}  & \srp{60}{6}      & \srp*{63}{5}    \\
  & \SNT & \srp{85}{33}  & \srp{97}{7}      & \srp*{100}{0}    \\
  & \SAT & \sa{1.7}{0.9} & \sa{1.5}{0.5} & \sa*{1.4}{0.4}\\
  & \SCT & \scos{14.8}{7.7} & \scos{13.1}{4.0} & \scos*{10.6}{2.7}\\
\cmidrule{1-5}
\multirow{3}*{14} 
  & \TTT & \srp{36}{7}  & \srp{53}{5}      & \srp*{55}{7}    \\
  & \SNT & \srp{71}{39}  & \srp{77}{35}      & \srp*{100}{0}    \\
  & \SAT & \sa{1.7}{1.2} & \sa*{1.3}{0.7} & \sa{1.7}{0.4}\\
  & \SCT & \scos{15.2}{10.7} & \scos*{11.5}{6.6} & \scos{13.4}{3.2}\\
\cmidrule{1-5}
\multirow{3}*{16} 
  & \TTT & \srp{31}{8}  & \srp*{51}{5}      & \srp*{51}{5}    \\
  & \SNT & \srp{56}{37}  & \srp{59}{39}      & \srp*{99}{3}    \\
  & \SAT & \sa{1.6}{1.3} & \sa*{1.1}{0.9} & \sa{1.8}{0.4}\\
  & \SCT & \scos{14.1}{12.1} & \scos*{10.0}{8.1} & \scos{14.2}{3.7}\\
\bottomrule
\end{tabular}}
\caption{\textbf{Mechanical Search in Sim.} Simulation results by SLAX-RAY over 900 trials on a $5\times4$ discretized shelf across 180 total object configurations. All averages are weighted by the retrieval probabilities of each object. Mean steps and cost are computed among successful searches only. \algabbr provides arrangements that lead to the highest proportion of requested targets being visible (i.e., directly retrievable) and achieves the highest search success rate when the target is hidden. The smaller mean steps and cost for baselines in higher density shelves are because SLAX-RAY is only able to find objects in the front rows of those arrangements while it can find objects at the back for \algabbr.}
\label{tab:results_slaxray_sim_54}
\vspace{-16pt}
\end{table}

%% file: tables/table2.tex
\newcommand{\NT}{\% Hidden Objects}
\newcommand{\TT}{\% Visible Objects}
\newcommand{\SAP}{Mean Steps Hidden}
\newcommand{\SCP}{Mean Cost Hidden}

\begin{table}[t]
\vspace{4pt}
\centering
\begin{tabular}{@{}ccccc@{}}\toprule
No. & Metric & Random  & Priority-Greedy & \algabbr \\

\midrule
\multirow{3}*{4} 
  & \TT & \sr{63.2}  & \sr{78.9}      & \sr*{100}    \\
  & \SAP & 1.0 & 1.0 & \bf{0.0} \\
  & \SCP & 7.0 & 10.0 & \bf{0.0} \\
\cmidrule{1-5}
\multirow{3}*{6} 
  & \TT & \sr{77.8}  & \sr{55.6}      & \sr*{81.5}    \\
  & \SAP & 4.7 & 2.6 & \bf{2.2}\\
  & \SCP & 37.7 & 20.1 & \bf{11.2}\\
\cmidrule{1-5}
\multirow{3}*{8} 
  & \TT & \sr{50.0}  & \sr*{69.4}      & \sr*{69.4}    \\
  & \SAP & 5.1 & 3.1 & \bf{2.4}\\
  & \SCP & 29.2 & 22.6 & \bf{11.1}\\

\bottomrule
\end{tabular}
\caption{\textbf{Mechanical Search in Real.} Results by SLAX-RAY over 54 trials with a Fetch robot on a physical shelf discretized into a $4\times3$ grid. As the shelf is not dense, SLAX-RAY finds all targets successfully in all arrangements. All averages are weighted by the objects' retrieval probabilities. \algabbr produces arrangements for which SLAX-RAY can find the target in 2x fewer steps and 3x lower cost than baselines, suggesting that mechanical search policies can benefit from optimally arranging objects for retrieval.}
\label{tab:results_physical_slaxray}
\vspace{-12pt}
\end{table}

%% file: includes/9-conclusion.tex
\section{Conclusion and Future Work}
In this paper, we formalize the Optimal Shelf Arrangement (OSA) problem, and propose \algabbr, a mixed-integer program that efficiently optimizes an upper bound of the expected retrieval cost and solves for a near-optimal arrangement. Experiments suggest that optimally arranged shelves achieve significant cost savings compared to randomly and greedily arranged shelves for both object retrieval and search. In future work, we will consider the problem of rearranging the current shelf into a near-optimal arrangement; given an arrangement, instead of completely rearranging the shelf to place the objects in an optimal arrangement, it may be desirable to find a near-optimal arrangement that also minimizes disturbance from the current configuration.

%% file: includes/10-acknowledgments.tex
\section*{Acknowledgements}
{\footnotesize
This research was performed at the AUTOLAB at UC Berkeley in affiliation with the Berkeley AI Research (BAIR) Lab, and the CITRIS ``People and Robots'' (CPAR) Initiative. The authors were supported in part by donations from Google, Siemens, Autodesk, Bosch, Toyota Research Institute, Autodesk, Honda, Intel, Hewlett-Packard and by equipment grants from PhotoNeo, NVIDIA, and Intuitive Surgical.}

%% file: appendix.tex
\section{Appendix}

\subsection{Fully-Linearized MIP Formulation for OSA}
In this section, we give the linearized formulation of the proposed MIP:

\begin{gather}
    \text {Minimize } \sum_{l=1}^{n} \sum_{i=1}^{m_x} \sum_{j=1}^{m_y} p_{l} (y_{l i j}+c_r b_{l i j}) \label{appendix_mip_obj_combined}\\
    \text{Subject to } \sum_{i=1}^{m_x} \sum_{j=1}^{m_y} x_{l i j}=1, \quad \forall l \label{appendix_mip_object}\\
    a_{i j} = \sum_{l=1}^{n} x_{l i j} \leq 1, \quad \forall i, j \label{appendix_mip_position}\\
    \delta_{i j} \leq M \left(2-(a_{i-1, j} + a_{i+1, j}) \right) \quad \forall i, j  \label{appendix_mip_pushing_indicator}\\
    \delta^p_{l i j} \leq x_{l i j} \quad \forall l, i, j  \label{appendix_mip_pushing_indicator_object_level_1}\\
    \delta^p_{l i j} \leq \delta_{i j} \quad \forall l, i, j  \label{appendix_mip_pushing_indicator_object_level_2}\\
    \delta^p_{l i j} \geq x_{l i j} + \delta_{i j} - 1 \quad \forall l, i, j  \label{appendix_mip_pushing_indicator_object_level_3}\\
    \delta^s_{l i j} \leq x_{l i j} \quad \forall l, i, j  \label{appendix_mip_suctioning_indicator_object_level_1}\\
    \delta^s_{l i j} \leq 1 - \delta_{l i j} \quad \forall l, i, j  \label{appendix_mip_suctioning_indicator_object_level_2}\\
    \delta^s_{l i j} \geq x_{l i j} - \delta_{l i j} \quad \forall l, i, j  \label{appendix_mip_suctioning_indicator_object_level_3}\\
    \sum_{l' \neq l} \sum_{j^{\prime}=1}^{j-1} \left[c_{l' p} \delta^p_{l' i j'} + c_{l' s} \delta^s_{l' i j'}\right] - y_{l i j} \leq M z_{l i j} \quad \forall l, i, j \label{appendix_mip_retrieval}\\
    x_{l i j} \leq 1 - z_{l i j}, \quad \forall i, j \label{auxiliary_x_negation}\\
    e_{i j} = \sum_{j^{\prime}=1}^{j} a_{i j'}, \quad \forall i, j \label{appendix_mip_no_obj_in_front}\\
    e_{i j} \leq M (1-f_{i j}) \quad \forall i, j \label{appendix_mip_empty_aisle}\\
    d_i = \sum_{j=1}^{m_y} f_{i j}, \quad \forall i \label{appendix_mip_depth_empty_space}\\
    \tilde \delta_{i j} \leq a_{i j}, \quad \forall i, j \label{appendix_mip_hiding_indicator_1}\\
    \tilde \delta_{i j} \leq g_{i j}, \quad \forall i, j \label{appendix_mip_hiding_indicator_2}\\
    \tilde \delta_{i j} \geq a_{i j} + g_{i j} - 1, \quad \forall i, j \label{appendix_mip_hiding_indicator_3}\\
    1 - h_{i j} \geq a_{i j}, \quad \forall i, j \label{auxiliary_h_1}\\
    1 - h_{i j} \geq f_{i j}, \quad \forall i, j \label{auxiliary_h_2}\\
    1 - h_{i j} \leq a_{i j} + f_{i j}, \quad \forall i, j \label{auxiliary_h_3}\\
    g_{i j} \geq h_{i-1, j}, \quad \forall i, j \label{auxiliary_g_1}\\
    g_{i j} \geq h_{i+1, j}, \quad \forall i, j \label{auxiliary_g_2}\\
    g_{i j} \leq h_{i-1, j} + h_{i+1, j}, \quad \forall i, j \label{auxiliary_g_3}\\
    \bigg(e_{i j} - x_{l i j} - \sum_{j^{\prime}=1}^{j-1} \tilde \delta_{i j'} - \sum_{i^{\prime} \neq i} d_{i'}\bigg) - b_{l i j} \leq M z_{l i j} \quad \forall l, i, j \label{appendix_mip_number_removal}\\
    x_{l i j}, z_{l i j}, f_{i j}, g_{i j}, h_{i j}, \delta_{i j}, \delta^p_{l i j}, \delta^s_{l i j}, \tilde \delta_{i j} \in\{0,1\},  \quad \forall l, i, j, \label{appendix_mip_domain1}\\
    y_{l i j}  \in \R_{\geq 0}, a_{i j}, b_{l i j}, d_i, e_{i j} \in \Z_{\geq 0}, \quad \forall l, i, j. \label{appendix_mip_domain2}
\end{gather}

\subsection{Analysis of Suboptimality for Figures~\ref{fig:misselection_of_mip} and \ref{fig:suboptimality_of_mip}}
In this section, we give a detailed analysis of the condition for the two examples in Section~\ref{sec:optimality_analysis} where \algabbr may find a suboptimal solution.

\subsubsection{Figure~\ref{fig:misselection_of_mip}}

Comparing arrangement (2b) and (2a), the difference lies in the retrieval costs for objects 5 and 6. For retrieving object 5, arrangement (2b) costs $\min(c_{4s}, c_{4p}+c_{7s})$ while arrangement (2a) costs only $c_{4p}$. For retrieving object 6, arrangement (2b) costs $\min(c_{4s}+c_{5p}, c_{7s}+c_{4p}+c_{5p})$ while arrangement (2a) costs $\min(c_{4p}+c_{5s}, c_{7s}+c_{4p}+c_{5p})$. If $\Delta c_4 \leq c_{7s}$ and $\Delta c_5 \geq c_{7s}$, then arrangement (2b) costs $\Delta c_4$ less for retrieving object 5 but $c_{7s} - \Delta c_4$ more for retrieving object 6 compared to arrangement (2a). Considering the retrieval probabilities, if $p_5 \Delta c_4 > p_6 (c_{7s} - \Delta c_4)$, then arrangement (2b) is more optimal than arrangement (2a).

However, as the \algabbr does not consider the possibility of preemptive actions, it overestimates the retrieval costs for arrangement (2b). In particular, for retrieving object 5, it will conclude that arrangement (2b) saves $\Delta c_4$, and for retrieving object 6, it will conclude that arrangement (2b) costs $(c_{4p} + c_{5s}) - (c_{4s} + c_{5p}) = \Delta c_5 - \Delta c_4$ more. Now if $p_6 (\Delta c_5 - \Delta c_4) > p_5 \Delta c_4 > p_6 (c_{7s} - \Delta c_4)$, \algabbr will choose arrangement (2a) over arrangement (2a) even though arrangement (2b) is actually more optimal if preemptive actions are allowed.

\subsubsection{Figure~\ref{fig:suboptimality_of_mip}}


Let the configuration in Figure 3a be $\mathcal{S}_A$, and the configuration in Figure 3b be $\mathcal{S}_B$. Then we have
\begin{equation}
\begin{aligned}
  &  \mathcal{C(S}_A) - \mathcal{C(S}_B)  &\\
   =& (p_4 + p_7) (-\min \left\{c_{12,s}, c_r \right\}) & \\
                                        &+ p_{10} (c_{9p} - \min \left\{c_{9s}, c_{9p}+c_{12,s} \right\}) & \\
                                        &+ p_{11} \big(\min \left\{c_{9p}+c_{10,s}, c_{9p}+c_{10,p}+c_{12,s} \right\} & \\
                                        &~~~~- \min \left\{c_{9s}+c_{10,p}, c_{9p}+c_{10,p}+c_{12,s} \right\} \big) &
\end{aligned}
\end{equation}
and
\begin{equation}
\begin{aligned}
&  \mathcal{\hat{C}(S}_A) - \mathcal{\hat{C}(S}_B) & \hspace{-2em}\hphantom{\text{Placeholder}}\\
                                        =& (p_4 + p_7) (-c_r) + p_{10} (c_{9p} - c_{9s}) &\\
                                        &+ p_{11} ((c_{9p}+c_{10,s}) - (c_{9s}+c_{10,p})). &
\end{aligned}
\end{equation}

When $\mathcal{C(S}_A) > \mathcal{C(S}_B)$ but $\mathcal{\hat{C}(S}_A) < \mathcal{\hat{C}(S}_B)$, the MIP will choose $\mathcal{S}_A$ over $\mathcal{S}_B$ while $\mathcal{S}_B$ is more optimal than $\mathcal{S}_A$ due to preemptive actions.

Since
\begin{equation}
\begin{aligned}
& \min \left\{c_{9p}+c_{10,s}, c_{9p}+c_{10,p}+c_{12,s} \right\} \\
-& \min \left\{c_{9s}+c_{10,p}, c_{9p}+c_{10,p}+c_{12,s} \right\} \\
=& \min \left\{\Delta c_{10}-\Delta c_9, c_{12,s}-\Delta c_9 \right\}
\end{aligned}
\end{equation}
if both components are nonnegative, we have $\mathcal{C(S}_A) > \mathcal{C(S}_B)$ when 
\begin{equation}
\begin{aligned}
& p_{11} \min \left\{\Delta c_{10}-\Delta c_9, c_{12,s}-\Delta c_9 \right\} \\
>& (p_4 + p_7) \min \left\{c_{12,s}, c_r \right\} + p_{10} \Delta c_9,
\end{aligned}
\end{equation}
assuming $\Delta c_{10}>\Delta c_9$ and $c_{12,s}>\Delta c_9$.

On the other hand, $\mathcal{\hat{C}(S}_A) < \mathcal{\hat{C}(S}_B)$ is equivalent to
\begin{align}
\begin{split}\label{eq:4}
  p_{11} (\Delta c_{10}-\Delta c_9) <{}& (p_4 + p_7) c_r + p_{10} \Delta c_9.
\end{split} 
\end{align}

Therefore, if $\Delta c_9 < c_{12,s} < c_r$, $\Delta c_{9}<\Delta c_{10}$, and $p_{11}$ is sufficiently large compared to $p_4 + p_7$ and $p_{10}$, then it is possible that
\begin{align}
\begin{split}\label{eq:5}
  {}&(p_4 + p_7) \min \left\{c_{12,s}, c_r \right\} + p_{10} \Delta c_9 \\
  <& p_{11} \min \left\{\Delta c_{10}-\Delta c_9, c_{12,s}-\Delta c_9 \right\} \\
  \leq& p_{11} (\Delta c_{10}-\Delta c_9) \\
  <& (p_4 + p_7) c_r + p_{10} \Delta c_9,
\end{split} 
\end{align}
in which case the \algabbr solution will be suboptimal.

\subsection{Runtime Analysis of \algabbr}

\begin{figure}
    \centering
    \includegraphics[width=0.85\linewidth]{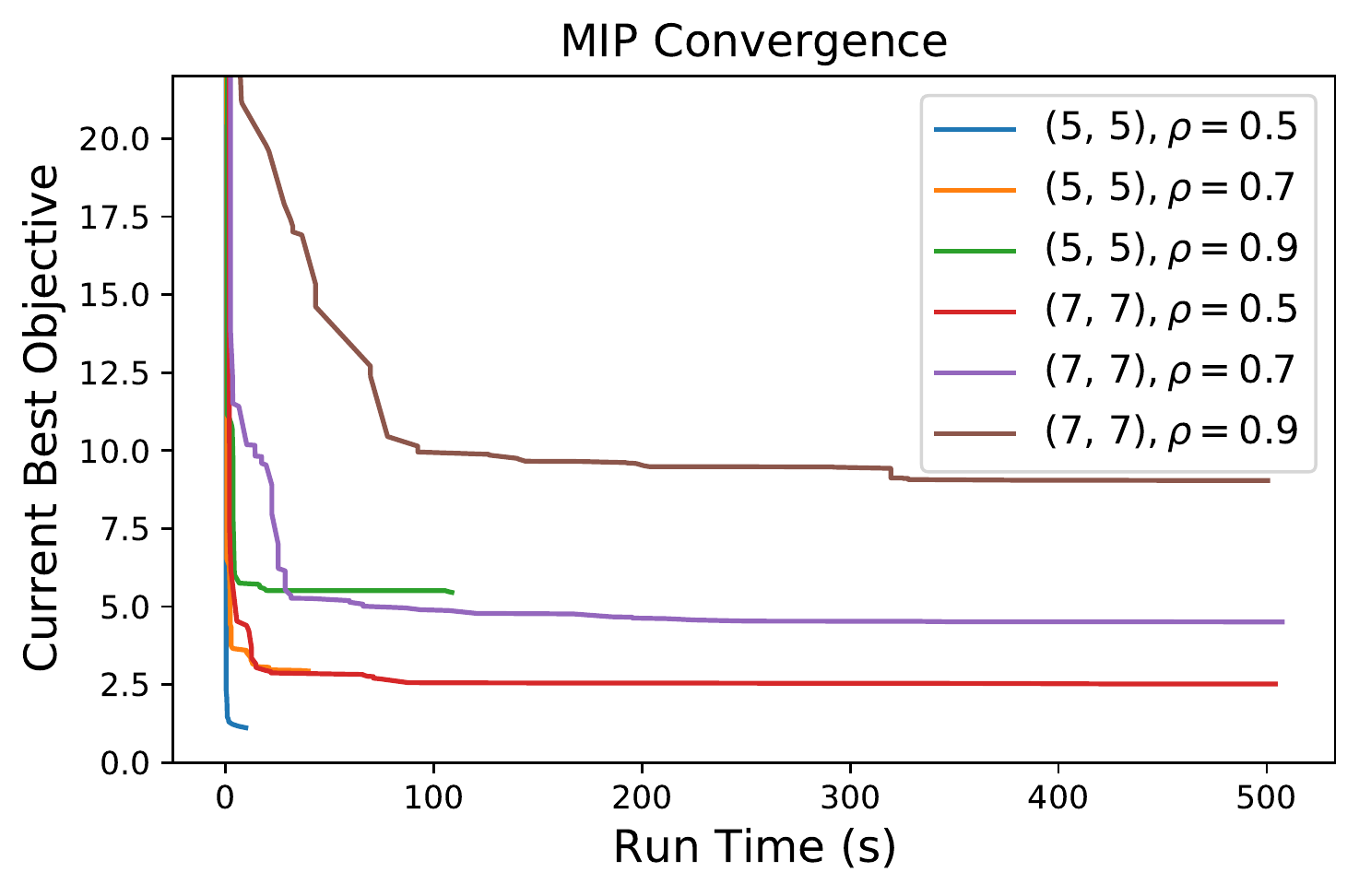}
    \includegraphics[width=0.85\linewidth]{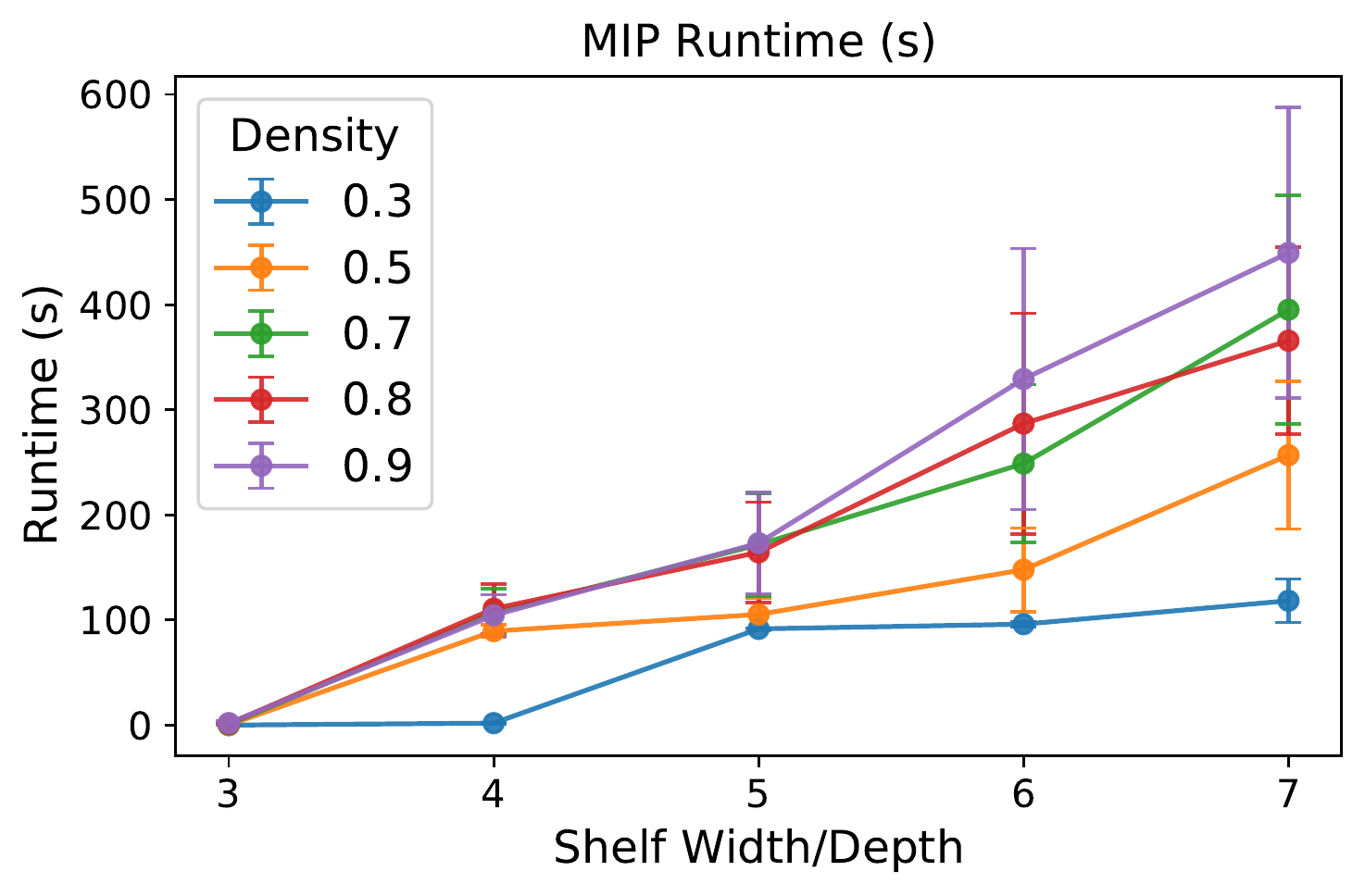}
    \caption{\algabbr runtime across shelf sizes and densities. Top: The objective of the best incumbent found during the process of optimization on 6 different problem instances (shelf width = 5 or 7, density = 0.3, 0.7, 0.9). The discontinuation of the lines indicates that a provably optimal solution has been found; otherwise, the optimization is run for 500s. Bottom: The MIP average runtime among the 675 problem instances (see Section~\ref{sec:sim_exp}), grouped by shelf width and density. The MIP is terminated if the best incumbent solution stops improving for 90s.
    }
    \vspace{-12pt}
    \label{fig:mip_runtime}
\end{figure}

Figure~\ref{fig:mip_runtime} plots the average computation time of \algabbr as a function of shelf size and density. The experiments were conducted using a Gurobi solver~\cite{gurobi} on a MacBook Pro with 2.4 GHz 8-Core Intel Core i9 and 32GB memory. The top figure plots the objective of the best incumbent found during the process of optimization on 6 different problem instances (shelf width = 5 or 7, density = 0.3, 0.7, 0.9). The discontinuation of the lines indicates that a provably optimal solution has been found; otherwise, the optimization is run for 500s. We can see that for 5x5 shelves, the solver is able to find an optimal solution within 120s even for 90\% density. For 7x7 shelves, the solver is not able to find a provably optimal solution within 500s, but the best objective has converged, indicating that a near-optimal solution has been found. Thus, for the 675 instances described in Section~\ref{sec:sim_exp}, we terminate the optimization if the best incumbent solution stops improving for 90s. The bottom figure plots the MIP average runtime grouped by shelf width and density. Unsurprisingly, the runtime scales with the shelf size and density. For a 7x7 shelf and 90\% density, 
it takes the solver about 8 minutes to find a near-optimal solution.

\subsection{More SLAX-RAY Simulation Results}

\input{tables/table1_43}

We show simulation results that use the same $4\times3$ discretization as in the physical experiments. We discretize the shelf into a $5\times4$ grid with 4, 6, 8, or 10 cuboid objects, with each setting repeated 50 times for each arrangement. The objects are generated with random dimensions within the limit of the grid size, and are assigned retrieval probabilities $p_l$, drawn from a standard uniform distribution and normalized. We set the pushing cost of each object to be proportional to its volume and use a cost ratio of $\psi=1.3$. 

Table~\ref{tab:results_slaxray_sim_43} shows the results. Similar to Table~\ref{tab:results_slaxray_sim_54}, the mean steps and cost are computed only among successful searches. We see that \algabbr arrangements not only have the largest proportion of time when the requested target is readily retrievable, but the search success rate among hidden objects remains 98-100\% even for relatively dense shelves, and is 45-58\% higher than baselines for 10 objects. The smaller mean steps and cost for baselines in dense shelves are because SLAX-RAY can only successfully find objects in the front rows which has fewer steps, while \algabbr leaves sufficient space for SLAX-RAY to manipulate and find deeply-hidden objects at the back so its weighted average steps and cost among all successful trials are larger.

%% file: tables/table1_43.tex
\begin{table}[t]
\vspace{4pt}
\centering
\resizebox{\columnwidth}{!}{
\begin{tabular}{@{}ccccc@{}}\toprule
No. & Metric & Random  & Priority-Greedy & \algabbr \\

\midrule
\multirow{3}*{4} 
  & \TTT & \srp{72}{20}  & \srp{84}{15}      & \srp*{100}{0}    \\
  & \SNT & \srp*{100}{0}  & \srp*{100}{0}      & \srp*{-}{-}    \\
  & \SAT & \sa{1.0}{0.1} & \sa{1.0}{0.1} & \sa*{-}{-}\\
  & \SCT & \scos{8.6}{2.0} & \scos{8.7}{1.9} & \scos*{-}{-}\\
\cmidrule{1-5}
\multirow{3}*{6} 
  & \TTT & \srp{58}{14}  & \srp{76}{10}      & \srp*{86}{5}    \\
  & \SNT & \srp{93}{23}  & \srp{99}{5}      & \srp*{100}{0}    \\
  & \SAT & \sa{1.1}{0.4} & \sa{1.1}{0.3} & \sa*{1.1}{0.5}\\
  & \SCT & \scos{10.1}{4.3} & \scos{9.3}{3.6} & \scos*{7.5}{6.3}\\
\cmidrule{1-5}
\multirow{3}*{8} 
  & \TTT & \srp{48}{11}  & \srp{67}{7}      & \srp*{73}{7}    \\
  & \SNT & \srp{83}{30}  & \srp{81}{35}      & \srp*{100}{0}    \\
  & \SAT & \sa{1.2}{0.6} & \sa{1.1}{0.6} & \sa*{1.0}{0.2}\\
  & \SCT & \scos{10.1}{5.7} & \scos{8.5}{4.7} & \scos*{7.5}{1.3}\\
\cmidrule{1-5}
\multirow{3}*{10} 
  & \TTT & \srp{40}{10}  & \srp{59}{7}      & \srp*{60}{7}    \\
  & \SNT & \srp{40}{35}  & \srp{53}{40}      & \srp*{98}{4}    \\
  & \SAT & \sa{0.6}{0.7} & \sa*{0.7}{0.6} & \sa{1.2}{0.2}\\
  & \SCT & \scos{5.1}{6.6} & \scos*{5.7}{5.4} & \scos{9.2}{2.2}\\
\bottomrule
\end{tabular}}
\caption{\textbf{Mechanical Search in Sim, $4\times3$ discretization.} Simulation results by SLAX-RAY over 600 trials on a $4\times3$ discretized shelf across 120 total object configurations. All averages are weighted by the retrieval probabilities of each object. Mean steps and cost are computed among successful searches only. \algabbr provides arrangements that lead to more objects that are directly visible and a search success rate among hidden objects 45-58\% higher than baselines. The smaller mean steps and cost for baselines in higher density shelves are because SLAX-RAY is only able to find objects in the front rows of those arrangements while it can find objects at the back of the shelf for \algabbr arrangements.}
\label{tab:results_slaxray_sim_43}
\vspace{-12pt}
\end{table}